\documentclass[11pt]{article}
\usepackage{mathrsfs}
\usepackage{amsmath}
\usepackage{amsfonts}
\usepackage{amssymb}
\usepackage{latexsym}
\usepackage[dvips]{graphicx}
\newtheorem{theorem}{\bf Theorem}[section]
\newtheorem{lemma}[theorem]{\bf Lemma}
\newtheorem{proposition}[theorem]{\bf Proposition}

\newenvironment{proof}{\noindent{\em Proof:}}{\quad \hfill$\Box$\vspace{2ex}}

\newtheorem{definition}[theorem]{\bf Definition}

\def \bN {\Bbb N}
\def \bZ {\Bbb Z}

\def \bF {\Bbb F}
\def \bR {\Bbb R}

\def \bC {\Bbb C}

\def \bF {\Bbb F}

\def \bP {\Bbb P}
\def \bQ {\Bbb Q}

\def \cB {{\cal B}}
\def \cD {{\cal D}}
\def \cF {{\cal F}}
\def \cH {{\cal H}}

\def \cL {{\cal L}}

\def \cS {{\cal S}}

\def \cR {{\cal R}}

\def \cZ {{\cal Z}}

\def \cW {{\cal W}}

\def \and {\, \mbox{\rm and}\, }

\def \span {\,{\rm span}\,}
\def \supp {\,{\rm supp}\,}

\makeatletter

\newcommand{\Rmnum}[1]{\expandafter\@slowromancap\romannumeral #1@}
\makeatother

\begin{document}

\title{Universalities of Reproducing Kernels Revisited}

\author{Benxun Wang\thanks{School of Mathematics and Computational
Science, Sun Yat-sen University, Guangzhou 510275, P. R. China. E-mail address: {\it wangbx3@mail2.sysu.edu.cn}.}\quad and\quad Haizhang Zhang\thanks{School of Mathematics and Computational
Science and Guangdong Province Key Laboratory of Computational
Science, Sun Yat-sen University, Guangzhou 510275, P. R. China. E-mail address: {\it zhhaizh2@mail.sysu.edu.cn}. Supported in part by Natural Science Foundation of China under grants 11222103 and 11101438, and by the US Army Research
 Office.}}
 \date{}
\maketitle

%
%

\begin{abstract}
Kernel methods have been widely applied to machine learning and other questions of approximating an unknown function from its finite sample data. To ensure arbitrary accuracy of such approximation, various denseness conditions are imposed on the selected kernel. This note contributes to the study of universal, characteristic, and $C_0$-universal kernels. We first give simple and direct description of the difference and relation among these three kinds of universalities of kernels. We then focus on translation-invariant and weighted polynomial kernels. A simple and shorter proof of the known characterization of characteristic translation-invariant kernels will be presented. The main purpose of the note is to give a delicate discussion on the universalities of weighted polynomial kernels.\newline

\end{abstract}

\section{Introduction}
\setcounter{equation}{0}
Many scientific questions can be mathematically formulated as the learning of an unknown function from its finite sample data. Suppose the unknown target function $f_0$ lives on the input space $X$ and its sample data on the finite sampling points $x_1,x_2,\cdots, x_n\in X$ are available. We human beings learn from experience. By this intuition, a predictor function learned from the finite sample data of $f_0$ on $x_1,x_2,\cdots,x_n$ should be of the form
\begin{equation}\label{section}
\sum_{j=1}^n c_j K(x_j,\cdot),
\end{equation}
where $c_j$'s are constants and $K$ is a function on $X\times X$ that measures the similarity between inputs from $X$.

The inner product is a natural mathematical tool of measuring similarity. By this consideration, the function $K$ in (\ref{section}) is chosen to be of the form
$$
K(x,y)=\langle \Phi(x),\Phi(y)\rangle_\cW,\ \ x,y\in X
$$
where $\Phi$ is a mapping from $X$ to a Hilbert space $\cW$ with inner product $\langle \cdot,\cdot\rangle_\cW$. It has been understood that a function $K$ on $X\times X$ has the above inner product representation if and only if it is a positive-definite function \cite{Aronszajn}, that is, if for all finite points $x_1,x_2,\cdots,x_n\in X$, the matrix
$$
[K(x_j,x_k)]_{j,k=1}^n
$$
is symmetric and positive semi-definite. Moreover, for every positive-definite function $K$ on $X\times X$ there exists a unique Hilbert space, denoted as $\cH_K$, of certain functions on $X$ such that $K(x,\cdot)\in \cH_K$ for all $x\in X$ and
\begin{equation}\label{reproducing}
f(x)=\langle f,K(x,\cdot)\rangle_{\cH_K}\mbox{ for all }f\in\cH_K, \ x\in X.
\end{equation}
By equation (\ref{reproducing}), for each $x\in X$, the point evaluation functional $f\to f(x)$ is a continuous on $\cH_K$. It implies that $\cH_K$ is a {\it reproducing kernel Hilbert space} on $X$. For the sake of (\ref{reproducing}), positive-definite functions are usually called {\it reproducing kernels} in machine learning.

A pleasant coincidence is that the minimizer for all feasible regularization learning algorithms in $\cH_K$ must have the form (\ref{section}). Specifically, for any continuous loss function $\cL:\bR^n\to[0,+\infty)$ and non-decreasing regularizer $\phi:[0,+\infty)\to [0,+\infty)$, every minimizer of
\begin{equation}\label{kernelmethod}
\inf_{f\in\cH_K}\cL(f(x_1),f(x_2),\cdots,f(x_n))+\phi(\|f\|_{\cH_K})
\end{equation}
is of the form (\ref{section}). The result is known as the {\it representer theorem} in machine learning \cite{KW}. The hypothesis error in the error estimate of regularization learning algorithms vanishes automatically due to the representer theorem, \cite{CuckerSmale,CuckerZhou}. Learning algorithms (\ref{kernelmethod}) are the typical {\it kernel methods} \cite{Bishop,Evgeniou,Sch} in machine learning. Summarizing the above discussion, we conclude that kernel methods have the natural interpretation of learning from experience and are based on the sound mathematical theory of reproducing kernel Hilbert spaces.

The possibility of approximating the unknown target function from functions of the form (\ref{section}) should be first addressed in kernel methods. This denseness question motivates extensive study on universalities of reproducing kernels \cite{Steinwart1,Micchelli,Sriper,Sriper3,Sriper1,CMPY,Carmeli}. Assume from now on that the input space $X$ is a metric space. We also denote for each compact subset $\cZ\subseteq X$ by $C(\cZ)$ the space of continuous functions on $\cZ$ equipped with the maximum norm
$$
\|f\|_{C(\cZ)}:=\max\{x\in \cZ:|f(x)|\},
$$
and denote by $C_0(X)$ the space of continuous functions on $X$ that vanish at infinity. The study of universal kernels was initiated by Steinwart \cite{Steinwart1}, who posed the question of whether the function in (\ref{section}) can approximate any continuous target function arbitrarily well on any compact subset of the input space as the number of sampling points increases. Apparently, this is possible if and only if the linear span of $\{K(x,\cdot):x\in \cZ\}$ is dense in $C(\cZ)$ for all compact $\cZ\subseteq X$. This leads to the definition of universal kernels in \cite{Micchelli}.

\begin{definition}\label{universalkernel}
Let $X$ be a metric space and $K$ a continuous reproducing kernel on $X$. We call $K$ a {\it universal kernel} on $X$ if for every compact subset $\cZ\subseteq X$, $\span\{K(x,\cdot):x\in\cZ\}$ is dense in $C(\cZ)$.
\end{definition}

Two other universalities of reproducing kernels appear in the study of reproducing kernel Hilbert space embedding of probability measures \cite{Sriper,Sriper3,Sriper1} and in the construction of reproducing kernel Banach spaces with the $\ell^1$-norm \cite{Song2011,Song2013}.
\begin{definition}\label{chaandc0universal}
Let $X$ be a metric space and $K$ a reproducing kernel on $X$ such that $K(x,\cdot)\in C_0(X)$ for all $x\in X$. We call $K$ a {\it characteristic kernel} if the mapping from the set of probability Borel measures on $X$ to $\cH_K$ given by
$$
\bP\to \int_X K(\cdot,t)d\bP(t)
$$
is injective. We call $K$ a {\it $C_0$-universal kernel} if $\span\{K(x,\cdot):x\in X\}$ is dense in $C_0(X)$.
\end{definition}

Characterizations of universal kernels and sufficient conditions for various reproducing kernels to be universal were provided in \cite{Steinwart1,Micchelli}. The obtained results have also been established for vector-valued reproducing kernels \cite{CMPY,Carmeli}. Characteristic and $C_0$-universal kernels have been extensively studied in \cite{Sriper,Sriper3,Sriper1}. This note endeavors to contribute to the study of universalities of reproducing kernels. Firstly, translation-invariant kernels on Euclidean spaces constitute an important class of reproducing kernels. It has been obtained in \cite{Sriper} that a continuous translation-invariant kernel is characteristic if and only if its Fourier transform is supported everywhere. As a first contribution of this note, we give a simple and shorter proof of this important result in Section 3. It was pointed both in \cite{Micchelli} and \cite{Sriper} that if the support of the Fourier transform of a continuous translation-invariant kernel is a uniqueness set for the set of all the entire functions then the kernel is universal. Our discussion in Section 3 will reveal that this sufficient condition is too strong and is not necessary. Another contribution of this note is that we shall give a detailed discussion on the universalities of polynomial kernels and weighted polynomial kernels, which were barely discussed in \cite{Micchelli} or \cite{Sriper,Sriper3,Sriper1}.

\section{Characterization of Universalites by Borel Measures}
\setcounter{equation}{0}

The purpose of this section is to introduce necessary preliminaries and characterizations of universalities of reproducing kernels for later use.

From now on, $X$ stands for a prescribed metric space and all the function spaces are over the field $\bR$ of real numbers. The Banach space $C_0(X)$ consists of all the continuous functions $f$ on $X$ with the property that for all $\varepsilon>0$, $\{x\in X:|f(x)|\ge \varepsilon\}$ is compact in $X$. The norm on $C_0(X)$ is also the maximum norm, that is, for all $f\in C_0(X)$,
$$
\|f\|_{C_0(X)}=\max\{|f(x)|:x\in X\}.
$$
Denote by $\cB(Y)$ the set of all the finite signed Borel measures on a metric space $Y$. The dual space of $C_0(X)$ is $\cB(X)$, \cite{Rudin1}. It implies that $T$ is a continuous linear functional on $C_0(X)$ if and only if there exists a Borel measure $\mu\in\cB(X)$ such that
$$
T(f)=\int_X f(t)d\mu(t),\ \ f\in C_0(X).
$$
For each compact subset $\cZ\subseteq X$, the dual of $C(\cZ)$ is also $\cB(\cZ)$.

The following characterization of denseness is a direct corollary of the Hahn-Banach theorem in functional analysis, \cite{Conway}.

\begin{lemma}\label{HahnBanach}
Let $\cB$ be a Banach space and $A\subseteq \cB$. Then the linear span of $A$ is dense in $\cB$ if and only if there does not exist a nontrivial continuous linear functional on $\cB$ that vanishes on $A$.
\end{lemma}

With the above lemma, one immediately obtains the following characterization of universalities of reproducing kernels.

\begin{theorem}\label{characterizationsthree}
Let $K$ be a reproducing kernel on $X$ such that $K(x,\cdot)\in C_0(X)$ for all $x\in X$. Then the followings hold true:
\begin{description}
\item[(i)] The kernel $K$ is universal on $X$ if and only if for each compact subset $\cZ\subseteq X$, there does not exist a nonzero Borel measure $\mu\in \cB(\cZ)$ such that
    $$
    \int_\cZ K(x,t)d\mu(t)=0\mbox{ for all }x\in \cZ.
    $$

\item[(ii)] The kernel $K$ is $C_0$-universal on $X$ if and only if there does not exist a nonzero Borel measure $\mu\in \cB(X)$ such that
    $$
    \int_X K(x,t)d\mu(t)=0\mbox{ for all }x\in X.
    $$

\item[(iii)] The kernel $K$ is characteristic on $X$ if and only if there does not exist a nonzero Borel measure $\mu\in \cB(X)$ such that $\mu(X)=0$ and
    \begin{equation}\label{killKinc0}
    \int_X K(x,t)d\mu(t)=0\mbox{ for all }x\in X.
    \end{equation}
\end{description}
\end{theorem}
\begin{proof}
Statements (i) and (ii) follow immediately from Lemma \ref{HahnBanach}. To confirm (iii), suppose first that there does not exist a nonzero Borel measure $\mu\in \cB(X)$ satisfying $\mu(X)=0$ and equation (\ref{killKinc0}). Assume that $\bP,\bQ$ are two probability Borel measures on $X$ such that
$$
\int_XK(x,t)d\bP(t)-\int_XK(x,t)d\bQ(t)=0\mbox{ for all }x\in X.
$$
Then we have $(\bP-\bQ)(X)=1-1=0$ and
$$
\int_X K(x,t)d(\bP-\bQ)(t)=0\mbox{ for all }x\in X.
$$
By our assumption, $\bP-\bQ=0$. Thus, $K$ is characteristic. To see the converse, suppose that $K$ is characteristic and $\mu\in\cB(X)$ satisfies $\mu(X)=0$ and equation (\ref{killKinc0}). By the Hahn-Jordan decomposition theorem (see, \cite{Dudley}, Theorem 5.6.1), there exist unique positive Borel measures $\mu^+$ and $\mu^-$ on $X$ such that $\mu=\mu^+-\mu^-$. As $\mu(X)=0$, $\mu^+(X)-\mu^-(X)=0$. Let $C=\mu^+(X)$. Then
 $$
 \mu=C(\mathbb{P}-\mathbb{Q}),
 $$
where $\bP=\mu^+/C$ and $\bQ=\mu^{-}/C$ are two probability Borel measures. By (\ref{killKinc0}) and by fact that $K$ is characteristic, $\mu^{+}=\mu^{-}$, yielding that $\mu=0$. The proof is complete.
\end{proof}

One can draw a few conclusions on the relations among the three kinds of universalities of reproducing kernels. By (ii) and (iii) in the above theorem, a $C_0$-universal kernel must be characteristic. By Urysohn's lemma in topology (see, for example, \cite{Munkres}, page 207), every continuous function on a compact subset of the metric space $X$ can be extended to a function in $C_0(X)$. Using this result, one sees from definitions \ref{universalkernel} and \ref{chaandc0universal} that a $C_0$-universal kernel must also be universal. For more discussion on the relations, see \cite{Sriper3,Sriper1}.

\section{Translation Invariant Kernels}
\setcounter{equation}{0}
Reproducing kernels $K$ on $\bR^d$ that are translation-invariant in the sense that
$$
K(x,y)=K(x+a,y+a)\mbox{ for all x,y,a}\in \bR^d
$$
are of particular importance in machine learning. The celebrated Bochner theorem \cite{Bochner2} asserts that $K$ is a continuous translation-invariant reproducing kernel on $\bR^d$ if and only if there exists a finite positive Borel measure $\nu$ on $\bR^d$ such that
\begin{equation}\label{bochnercha}
K(x,y)=\int_{\bR^d}e^{i(x-y)^T\xi}d\nu(\xi),\ \ x,y\in\bR^d.
\end{equation}
Let $K$ be given by (\ref{bochnercha}). An important result obtained in the studies \cite{Sriper} of reproducing kernel Hilbert spaces embedding of probability measure is that $K$ is characteristic if and only if $\supp\nu=\bR^d$. Recall that a point $x_0$ belongs to the support $\supp\nu$ of a positive Borel measure $\nu$ if for any open subset $V\subseteq \bR^d$ that contains $x_0$, $\nu(V)>0$. A main purpose of this section is to give a shorter proof of this result. Moreover, it was both noted in \cite{Sriper} and \cite{Micchelli} that when $\supp\nu$ is a uniqueness set for all the entire functions on $\bC^d$ then $K$ is universal. Another objective of this section is to present delicate discussion on the condition for $K$ defined by (\ref{bochnercha}) to be universal. In particular, one shall see that the uniqueness condition is not necessary.

To fulfill the two purposes, we shall need basic facts on distributions and the Fourier transform (see, for example, \cite{Fourier}). Denote by $\cD(\bR^d)$ the space of infinitely differentiable functions on $\bR^d$ with compact support and by $\cS(\bR^d)$ the Schwartz class of rapidly decreasing infinitely differentiable functions on $\bR^d$. There hold the relations that for all $p\in[1,+\infty]$,
$$
\cD(\bR^d)\subseteq\cS(\bR^d)\subseteq L^p(\bR^d).
$$
The dual spaces of $\cD(\bR^d)$ and $\cS(\bR^d)$ are denoted by $\cD'(\bR^d)$ and $\cS'(\bR^d)$, respectively. Elements in $\cD'(\bR^d)$ are called distributions. In particular, those in $\cS'(\bR^d)\subseteq \cD'(\bR^d)$ are called {\it tempered distributions}.

The Fourier transform $\hat{f}$ and the inverse Fourier transform $\check{f}$ of $f\in L^1(\bR^d)$ are respectively defined by
$$
\hat{f}(\xi):=\frac{1}{(2\pi)^{d/2}}\int_{\bR^d}e^{-ix^T\xi}f(x)dx, \ \ \xi\in\bR^d
$$
and $\check{f}(\xi)=\hat{f}(-\xi)$, $\xi\in\bR^d$. For all $\varphi\in\cS(\bR^d)$, $\hat{\varphi}\in\cS(\bR^d)$. Furthermore, the Fourier transform is continuous on $\cS(\bR^d)$. This allows us to define the Fourier transform on tempered distributions as follows:
$$
\hat{T}(\varphi):=T(\hat{\varphi}),\ \ T\in\cS'(\bR^d),\ \varphi\in\cS(\bR^d).
$$
The Fourier transform of a tempered distribution remains a tempered distribution. In particular, a signed Borel measure $\mu\in\cB(\bR^d)$ corresponds to a tempered distribution defined by
\begin{equation}\label{distributionmu}
\mu(\varphi):=\int_{\bR^d}\varphi(t)d\mu(t),\ \ \varphi\in\cS(\bR^d).
\end{equation}
It can be verified by the Fubini theorem that the Fourier transform of this distribution is
$$
\hat{\mu}(\xi)=\int_{\mathbb{R}^d}e^{-ix^T\xi}d\mu(x),\ \ \xi\in\bR^d,
$$
which is bounded and uniformly continuous on $\mathbb{R}^d$.

As a final preparation, we make the following simple observation.

\begin{lemma}\label{reducetosupport}
Let $K$ be a translation-invariant kernel on $\bR^d$ given by (\ref{bochnercha}), $\cZ$ a Borel subset in $\bR^d$ and $\mu\in \cB(\cZ)$. Then
\begin{equation}\label{reducetosupporteq1}
\int_\cZ K(x,t)d\mu(t)=0\mbox{ for all }x\in\cZ
\end{equation}
if and only if
\begin{equation}\label{reducetosupporteq2}
\hat{\mu}(\xi)=\int_{\cZ}e^{-ix^T\xi}d\mu(x)=0\mbox{ for all }\xi\in\supp\nu.
\end{equation}
\end{lemma}
\begin{proof}
Suppose first that (\ref{reducetosupporteq2}) holds true. By Fubini's theorem, we get for all $x\in \cZ$ that
$$
\int_\cZ K(x,t)d\mu(t)=\int_\cZ \int_{\bR^d}e^{i(x-t)^T\xi}d\nu(\xi)d\mu(t)=\int_{\bR^d}e^{ix^T\xi}d\nu(\xi)\int_\cZ e^{-it^T\xi}d\mu(t)=\int_{\supp\nu}e^{ix^T\xi}\hat{\mu}(\xi)d\nu(\xi)=0.
$$
Conversely, suppose that (\ref{reducetosupporteq1}) is true. Then for all $x\in\cZ$,
$$
\int_{\supp\nu}e^{ix^T\xi}\hat{\mu}(\xi)d\nu(\xi)=0.
$$
Integrating both sides of the above equation with respect to $d\bar{\mu}(x)$ on $x\in\cZ$ yields by the Fubini theorem that
$$
\int_{\supp\nu}|\hat{\mu}(\xi)|^2d\nu(\xi)=0,
$$
which implies that $\mu$ vanishes everywhere on $\supp\nu$.
\end{proof}

\begin{theorem}\label{characterization6}
Let $K$ be the translation-invariant kernel on $\bR^d$ given by (\ref{bochnercha}). Then $K$ is characteristic if and only if $\supp\nu=\mathbb{R}^d$.
\end{theorem}
\begin{proof}
Suppose first that $\supp\nu=\bR^d$ and suppose that $\mu\in\cB(\bR^d)$ satisfies
$$
\int_{\bR^d} K(x,t)d\mu(t)=0
$$
for all $x\in\bR^d$. Letting $\cZ=\bR^d$ in Lemma \ref{reducetosupport} yields that $\hat{\mu}$ is the zero function. Thus, $\mu$ is the zero measure. By (iii) in Theorem \ref{characterizationsthree}, $K$ is characteristic.

Conversely, suppose that $K$ is characteristic but $\supp\nu$ is a proper subset of $\bR^d$. Then $U:=\mathbb{R}^d\setminus(\supp\nu\cup\{0\})$ is a non-empty open set. There hence exists a nontrivial function $\phi\in \cD(\bR^d)$ with $\supp\phi\subseteq U$. Let $f:=\check{\phi}$, and define a Borel measure $\mu$ on $\mathbb{R}^d$ by $\mu(A)=\int_A f(x)dx$. Clearly, as $\phi\in\cS(\bR^d)$, $f\in \cS(\bR^d)\subseteq L^1(\bR^d)$. Thus, $\mu$ belonds to $\cB(\bR^d)$ and is nontrivial. We hence have $\hat{\mu}=\hat{f}=\phi$, which vanishes on $\supp\nu$ and $0$. The latter implies $\mu(\bR^d)=0$. By Lemma \ref{reducetosupport} and (iii) in Theorem \ref{characterizationsthree}, $K$ is not a characteristic kernel, a contradiction.
\end{proof}

By the proof of the sufficiency above and that a $C_0$-universal kernel must be characteristic, $K$ defined by (\ref{bochnercha}) is $C_0$-universal if and only if $\supp\nu=\bR^d$. This has also been proved in \cite{Sriper3}.

We next turn to conditions for $K$ given by (\ref{bochnercha}) to be a universal kernel. The concept of uniqueness sets is needed.

\begin{definition}\label{uniqueness}
Let $\bF$ be a class functions on a set $\Omega$. A subset $A\subseteq\Omega$ is called a {\it uniqueness set } for $\bF$ if a function in $\cF$ vanishes on $A$ then it must vanish everywhere on $\Omega$.
\end{definition}

Denote by $\cB_c(\bR^d)$ the class of all finite signed Borel measures on $\bR^d$ whose support is compact. Set
$$
\cF(\cB_c(\bR^d)):=\{\hat{\mu}:\mu\in\cB_c(\bR^d)\}.
$$

By Lemma \ref{reducetosupport} and (i) in Theorem \ref{characterizationsthree}, we get the following characterization of universal kernels.

\begin{lemma}\label{lemmauniversal}
Let $K$ be defined by (\ref{bochnercha}). Then $K$ is a universal kernel on $\bR^d$ if and only if $\supp\nu$ is a uniqueness set for $\cF(\cB_c(\bR^d))$.
\end{lemma}

Note that for each $\mu\in\cB_c(\bR^d)$, $\hat{\mu}$ is the restriction on $\bR^d$ of an entire function on $\bC^d$ defined by
\begin{equation}\label{entirefunction}
\hat{\mu}(z):=\int_{\supp\mu} e^{-iz^Tt}d\mu(t),\ \ z\in\bC^d.
\end{equation}
From this observation, it was immediately concluded in \cite{Micchelli} and \cite{Sriper} that if $\supp\nu$ is a uniqueness set for all the entire functions on $\bC^d$ then $K$ is universal. We shall point out that this is unnecessary essentially by the observation that for each compactly-supported signed Borel measure $\mu$ on $\bR^d$, the function (\ref{entirefunction}) is not a general entire function but an entire function of exponential type, that is,
$$
|\hat{\mu}(z)|\le Ce^{\lambda\|z\|},\ \ z\in\bC^d
$$
for some positive constants $C,\lambda$. Here, $\|\cdot\|$ is the standard Euclidean norm. For detailed discussion, we introduce the completeness radius of complex exponentials. Set $B_R:=\{x\in\bR^d:\|x\|\le R\}$ for $R\ge 0$. The completeness radius of $\supp\nu$ is defined by
$$
\cR(\nu):=\sup\{R\ge0: \span\{e^{-ix^Tt}: t\in\supp\nu\}\mbox{ is dense in }C(B_R)\}.
$$

\begin{theorem}\label{radius}
The kernel $K$ given by (\ref{bochnercha}) is universal if and only if $\cR(\nu)=+\infty$.
\end{theorem}
\begin{proof}
The result follows directly from Lemma \ref{lemmauniversal}.
\end{proof}

We next restrict to the one-dimensional case. By the Weierstrass factorization theorem, if $\supp\nu$ has a finite accumulation point then it is a uniqueness set for all the entire functions on $\bC$. Consequently, $K$ is universal in this case. When $\supp\nu$ has no finite accumulation points, reference \cite{Beurling} provides a deep characterization of the completeness radius of $\supp\nu$ in terms of the Beurling-Malliavin density of the following measure
$$
\tilde{\nu}(A):=\#\{A\cap\supp\nu\}.
$$
Interested readers are referred to \cite{Beurling} for the detailed definition of the Beurling-Malliavin density, and to \cite{Redheffer} for an extensive survey on the completeness radius of complex exponentials. We conclude that in the one-dimensional case, $K$ defined by (\ref{bochnercha}) is a universal kernel on $\bR$ if and only if $\supp\nu$ has a finite accumulation point or the Beurling-Malliavin density of $\tilde{\nu}$ is infinite.

To end this section, we give an explicit example to show that the uniqueness condition in \cite{Micchelli,Sriper} is unnecessary. Let $\supp\nu:=\{\lambda_n:n\in\bN\}$ be free of finite accumulation points. Thus, $\supp\nu$ is not a uniqueness set for all the entire functions on $\bC$.

\begin{lemma}\cite{Redheffer}
Let $\supp\nu:=\{\lambda_n:n\in\bN\}$. If
\begin{equation}\label{lowerbound}
\limsup_{n\to+\infty}\frac{n}{|\lambda_n|}=+\infty
\end{equation}
then $\cR(\nu)=+\infty$.
\end{lemma}

Our example is explicitly given as
\begin{equation}\label{example}
\supp\nu:=\left\{\lambda_n=\frac {n}{\log(n+1)}:n\in
\bN\right\}
\end{equation}
and
$$
\nu(\lambda_n)=\frac1{n^2\log(n+1)},\ \ n\in\bN.
$$
Clearly, $\supp\nu$ has no finite accumulation points and (\ref{lowerbound}) is satisfied. As a result, $K$ is a universal kernel while $\supp\nu$ is not a uniqueness set for all the entire functions on $\bC$. Of course, there exist many other examples. For instance, $\supp\nu=\{n^{\lambda}:n\in\bN\}$ where $0<\lambda<1$.

\section{Polynomial Kernels}

In this section, we consider another important class of reproducing kernels--polynomial kernels. They are particular examples of the general Hilbert-Schimidt kernels
\begin{equation}\label{Hilbert}
K(x,y)=\sum_{n\in I}\phi_n(x)\phi_n(y),\ \ (x,y)\in X\times X,
\end{equation}
where $I$ is a countable index set, $\{\phi_n:n\in I\}\subseteq C(X)$, and the series converges pointwise on $X\times X$. By the Mercer theorem \cite{Mercer}, every continuous kernel is a Hilbert-Schmidt kernel.

We start with conditions ensuring a Hilbert-Schmidt kernel to be universal.

\begin{lemma}\label{Hilbert0}
Let $K$ be a Hilbert-Schmidt kernel given by (\ref{Hilbert}). Then the followings hold true:
\begin{description}
\item[(i)] Suppose the series in (\ref{Hilbert}) converges uniformly on every compact subset of $X\times X$. Then for each compact subset $\cZ\subseteq X$ and $\mu\in \cB(\cZ)$,
\begin{equation}\label{Hilbert0eq1}
\int_\cZ K(x,y)d\mu(y)=0\mbox{ for all } x\in \cZ
\end{equation}
if and only if
\begin{equation}\label{Hilbert0eq2}
\int_\cZ \phi_n(y)d\mu(y)=0\mbox{ for all } n\in I.
\end{equation}

\item[(ii)] Suppose there exists a nonnegative sequence $\lambda_n$, $n\in I$ such that
\begin{equation}\label{cohilbertcond}
|\phi_n(x)|\le \lambda_n\mbox{ for all }x\in X,\ n\in I\mbox{ and }\sum_{n\in I}\lambda_n<+\infty.
\end{equation}
Then for every $\mu\in\cB(X)$,
$$
\int_X K(x,y)d\mu(y)=0\mbox{ for all } x\in X
$$
if and only if
$$
\int_X \phi_n(y)d\mu(y)=0\mbox{ for all } n\in I.
$$
\end{description}
\end{lemma}
\begin{proof}
We prove (i) first. Let $\cZ\subseteq X$ be compact and $\mu\in\cB(\cZ)$. Suppose that (\ref{Hilbert0eq1}) holds true. For fixed $x\in \cZ$, by the uniform convergence of the series in (\ref{Hilbert}) on $\cZ$, we have
\begin{equation}\label{Hilbert0eq3}
\sum_{n\in I}\phi_n(x)\int_\cZ\phi_n(y)d\mu(y)=\int_\cZ\sum_{n\in I}\phi_n(x)\phi_n(y)d\mu(y)=\int_\cZ K(x,y)d\mu(y)=0.
\end{equation}
Integrating both sides of the above equation on $x\in \cZ$ with respect to $d\bar{\mu}(x)$ yields
$$
\biggl|\int_\cZ\phi_n(y)d\mu(y)\biggr|^2=0,\ \ n\in I
$$
which proves (\ref{Hilbert0eq2}). Conversely, if (\ref{Hilbert0eq2}) is true then (\ref{Hilbert0eq1}) follows immediately from (\ref{Hilbert0eq3}).

Statement (ii) can be proved in a similar way. One only needs to note that condition (\ref{cohilbertcond}) ensures that
$$
\int_X\sum_{n\in I}\phi_n(x)\phi_n(y)d\mu(y)=\sum_{n\in I}\phi_n(x)\int_X \phi_n(y)d\mu(y)
$$
and
$$
\int_X\sum_{n\in I}\phi_n(x)\int_X \phi_n(y)d\mu(y)d\bar{\mu}(x)=\sum_{n\in I}\biggl|\int_X \phi_n(y)d\mu(y)\biggr|^2.
$$
The proof is hence complete.
\end{proof}

As a direct corollary of the above result, we reprove the following characterizations of universal and $C_0$-universal Hilbert-Schmidt kernels in \cite{Micchelli} and \cite{Sriper3}.

\begin{proposition}\label{characterizationHilbert}
Let $K$ be given by (\ref{Hilbert}). Under the conditions in Lemma \ref{Hilbert0}, the followings hold true:
\begin{description}
\item[(i)] $K$ is universal on $X$ if and only if $\span\{\phi_n:n\in I\}$ is dense in $C(\cZ)$ for all compact $\cZ\subseteq X$,

\item[(ii)] $K$ is $C_0$-universal on $X$ if and only if $\span\{\phi_n:n\in I\}$ is dense in $C_0(X)$,

\item[(iii)] $K$ is characteristic on $X$ if and only if there does not exist a nonzero measure $\mu\in \cB(X)$ such that $\mu(X)=0$ and
$$
\int_X \phi_n(x)d\mu(x)=0\mbox{ for all }n\in I.
$$
\end{description}
\end{proposition}

In the rest of the section, we restrict our discussion to one-dimensional polynomial kernels. Let $\bZ_+:=\bN\cup\{0\}$ and denote for each sequence $\alpha:=\{\alpha_n\in\bR:n\in\bZ_+\}$ by $\supp\alpha:=\{n\in\bZ_+:\alpha_n\ne0\}$.

\begin{proposition}
Let $\alpha:=\{\alpha_n\ge0:n\in\bZ_+\}$ be such that the convergence radius of
$$
\sum_{n=0}^\infty\alpha_n z^n,\ \ z\in\bC
$$
is infinite. Then the polynomial kernel
$$
K(x,y):=\sum_{n=0}^\infty \alpha_n x^ny^n,\ \ x,y\in\bR
$$
is universal on $\bR$ if and only if $\alpha_0>0$ and
$$
\sum_{n\in 2\bN\cap\supp\alpha}\frac1n=\sum_{n\in (2\bN+1)\cap\supp\alpha}\frac1n=+\infty.
$$
\end{proposition}
\begin{proof}
The result follows from the celebrated M\"{u}ntz theorem \cite{Pinkus} that for $0=\lambda_0<\lambda_1<\lambda_2<\cdots\to\infty$, $\span\{x^{\lambda_n}:n=0,1,\cdots\}$ is dense in $C[0,1]$ if and only if
$$
\sum_{n=1}^\infty \frac1{\lambda_n}=+\infty
$$
and follows from the observation that each continuous function on $C[-1,1]$ can be factored into the sum of an even continuous function and an odd continuous function.
\end{proof}

We remark that the above characterization of universal polynomial kernels can be extended to higher-dimensional spaces by using extensions of the M\"{u}ntz theorem in multivariables (see, for example, \cite{Borwein,Ogawa}).

We next discuss characteristic and $C_0$-universality. As polynomials are  unbounded on non-compact domains, we shall consider the weighted polynomial kernels:
\begin{equation}\label{weightedpolynomial}
K_\omega(x,y):=\sum_{n=0}^{+\infty} \alpha_n\omega(x)x^n\omega(y)y^n,\ \ x,y\in\bR,
\end{equation}
where $\omega$ is a nonnegative continuous function on $\bR$ ensuring that $\omega x^n\in C_0(\bR)$ for all $n\in\bZ_+$. We also assume that there exists a nonnegative sequence $\{\lambda_n:n\in\bZ_+\}$ such that $\sum_{n=0}^\infty \lambda_n$ converges and
$$
\sqrt{\alpha_n}\omega(x)|x^n|\le \lambda_n\ \ \mbox{ for all }x\in\bR\mbox{ and }n\in\bZ_+.
$$
Thus, the weighted polynomial kernel (\ref{weightedpolynomial}) satisfies the condition of Lemma \ref{Hilbert0}. For $K_\omega$ to be a characteristic or $C_0$-universal kernel, by Proposition \ref{characterizationHilbert}, $\span\{\omega x^n:n\in\bZ_+\}$ must be dense in $C_0(\bR)$. The necessary and sufficient condition for this density has been established in \cite{Pollard}.

\begin{lemma}\label{Pollard}
The linear span of $\{\omega x^n:n\in\bZ_+\}$ is dense in $C_0(\bR)$  if and only if the following three conditions are satisfied simultaneously:
\begin{enumerate}
\item $\omega(x)\neq0\mbox{ for all }x\in\bR$,
\item $\int_\mathbb{R}\frac{\log\omega(x)}{1+x^2}dx=-\infty$,
\item there exists a sequence of polynomials $p_n$ and a constant $C$ such that $\lim_{n\rightarrow\infty}p_n(x)\omega(x)=1$ and $|p_n(x)\omega(x)|\leq C$ for all $x\in\bR$ and $n\in\bZ_+$.
\end{enumerate}
\end{lemma}

Let $\omega$ satisfy the three conditions in Lemma \ref{Pollard}. A typical example is $\omega(x)=e^{-x^2}$ with $p_n$ in condition 3 being
$$
p_n(x):=\sum_{k=0}^n \frac{x^{2k}}{k!},\ \ x\in\bR.
$$
By Proposition \ref{characterizationHilbert}, $K_\omega$ is both characteristic and $C_0$-universal if $\supp\alpha=\bZ_+$. We show below that this full support condition is not necessary.

\begin{theorem}
Let $K_\omega$ be defined by (\ref{weightedpolynomial}) and let $\omega$ be an even function on $\bR$ that is non-increasing on $[0,+\infty)$ and satisfy the three conditions in Lemma \ref{Pollard}. If the set $\bZ\setminus \supp\alpha$ is finite then $K_\omega$ is characteristic. If $\alpha_0>0$ and $\bZ\setminus \supp\alpha$ is finite then $K_\omega$ is $C_0$-universal.
\end{theorem}
\begin{proof}
Suppose first that $\bZ\setminus \supp\alpha$ is finite. Assume that $\mu\in\cB(\bR)$ satisfies
\begin{equation}\label{killsuppalpha}
\int_\bR \omega(x)x^nd\mu(x)=0\ \ \mbox{ for all }n\in\supp\alpha.
\end{equation}
Since $\omega$ is even on $\bR$ and non-increasing on $[0,+\infty)$, by the second condition in Lemma \ref{Pollard}, there exist positive constants $C$ and $A$ such that
$$
\omega(x)\le e^{-C(1+x^2)},\ \ |x|>A.
$$
Therefore, the function
$$
F(z):=\int_\bR \omega(x)e^{ixz}d\mu(x),\ \ z\in\bC
$$
is entire on $\bC$. Equation (\ref{killsuppalpha}) implies that
$$
F^{(n)}(0)=0\ \ \mbox{ for all }n\in\supp\alpha.
$$
As $\bZ\setminus \supp\alpha$ is finite, $F$ must be a polynomial. The distribution
$$
T(\varphi):=\int_\bR \omega(x)\varphi(x)d\mu(x),\ \ \varphi\in\cS(\bR)
$$
is hence a finite linear combination of the $\delta$ distribution and its derivatives. In other words, there exists some $m\in\bN$ and constants $c_j\in\bR$, $0\le j\le m$ such that
$$
\int_\bR \omega(x)\varphi(x)d\mu(x)=\sum_{j=0}^mc_j \varphi^{(j)}(0)\ \ \mbox{ for all }\varphi\in\cS(\bR).
$$
It implies that $\supp\mu$ is supported on the singleton $\{0\}$. As $\mu$ is a Borel measure, it must be a multiple of the delta density. When $\mu(\bR)=0$, we get $\mu=0$. When $\alpha_0>0$, we obtain from the first condition in Lemma \ref{Pollard} and equation (\ref{killsuppalpha}) that $\mu=0$. By Proposition \ref{characterizationHilbert}, the proof is complete.
\end{proof}

\bibliographystyle{amsplain}

\end{document}